\documentclass[letterpaper, 10pt, conference]{ieeeconf} 
\IEEEoverridecommandlockouts                              % This command is only needed if 
\usepackage{mathptmx} % assumes new font selection scheme installed
\usepackage{times} % assumes new font selection scheme installed
\usepackage{amsmath} % assumes amsmath package installed
\usepackage{amssymb}  % assumes amsmath package installed
\usepackage{tikz}
\usetikzlibrary{shapes, automata,fit,backgrounds}
\usepackage{wasysym}
\usepackage{paralist}
\usepackage{multirow}
\usepackage{graphicx}
\usepackage{cite}

\newcommand\vx{\mathbf{x}}
\newcommand\true{\mathit{true}}
\newcommand\false{\mathit{false}}
\newcommand\nx{\mathsf{next}}
\newtheorem{definition}{Definition}
\newtheorem{example}{Example}
\newtheorem{lemma}{Lemma}
\newcommand\scr[1]{\mathcal{#1}}
\newcommand\cost{\scr{C}}
\newcommand\hole{\tikz[anchor=base,baseline] \draw (0,0) node[fill=gray!20, draw=black]{?}; }
\newcommand{\Pp}{\mathbb{P}}

\begin{document}
\thispagestyle{empty}
\pagestyle{empty}

\title{Predictive Runtime Monitoring for Mobile Robots using Logic-Based Bayesian Intent Inference}
\author{Hansol Yoon and Sriram Sankaranarayanan~\thanks{Department of Computer Science, University of Colorado Boulder, USA, {\tt\small \{firstname.lastname\}@colorado.edu}}}

\maketitle

%%%%%%%%%%%%%%%%%%%%%%%%%%%%%%%%%%%%%%%%%%%%%%%%%%%%%%%%%%%%%%%%%%%%%%%%%%%%%%%%
\begin{abstract}
  We propose a predictive runtime monitoring framework that forecasts the
  distribution of future positions of mobile robots in order to detect and avoid  impending property violations such as collisions with obstacles or other agents. Our approach uses a restricted class of temporal logic formulas to   represent the likely intentions of the agents along with a combination of  temporal logic-based optimal cost path planning and Bayesian inference to
  compute the probability of these intents given the current trajectory of the
  robot. First, we construct a large but finite hypothesis space of possible intents  represented as temporal logic formulas whose atomic propositions are derived from a detailed map of the robot's workspace. 
  Next, our approach uses real-time observations
  of the robot's position to update a distribution over temporal logic formulae
  that represent its likely intent. This is performed by using a combination of  optimal cost path
  planning and a Boltzmann noisy rationality model. In this manner, we construct a Bayesian approach to evaluating the posterior probability of various hypotheses given the observed states and actions  of the robot. Finally, we predict the  future position of the robot by drawing posterior predictive samples using a Monte-Carlo method. We evaluate our framework using two different trajectory  datasets that contain multiple scenarios implementing various  tasks. The results show that our method can predict future positions precisely and efficiently, so that the computation time for generating a prediction is a tiny fraction of the overall time
  horizon.
\end{abstract}

%%%%%%%%%%%%%%%%%%%%%%%%%%%%%%%%%%%%%%%%%%%%%%%%%%%%%%%%%%%%%%%%%%%%%%%%%%%%%%%%
\section{INTRODUCTION}\label{sec:introduction}
Detecting and preventing imminent property violations is an important problem
for the safe operation of autonomous robots in highly dynamic environments. Such
violations include collisions between multiple robots, failure to respond to
events or robots entering restricted areas. 
Detecting such violations at design time is often impractical: behaviors are dependent 
on possible environmental conditions. The  space of
possible behaviors is too large, or may not be completely known to the
designers. Thus, runtime monitoring approaches have recently gained popularity. However,
these approaches require a model of the robot's motion to predict its future
position. Recent approaches have employed such models to detect the possible
positions that a robot can reach in the near future using physics-based dynamic
models and reachability analysis~\cite{althoff2014online,liu2017provably,
koschi2020set, chou2020predictive}. Similarly, a pattern-based approach predicts
future positions based on historical data and predicts likely future
positions~\cite{peddi2020datadriven} (Cf. ~\cite{rudenko2020human} for a survey
of trajectory prediction for dynamic agents).

However, forecasting future moves by extrapolating the past trajectories is
often likely to fail unless we also have a specification of the task (or current
subtask) that a robot is performing. In this paper, we term this as the robot's
(current/short term) \emph{intent}. This approach assumes that the robot has
a high-level mission or intent. Furthermore, the robot is assumed to 
choose an ``efficient'' plan for implementing the mission. The efficient path
can be either the shortest path or the \emph{almost} shortest path. In many
scenarios, this assumption is reasonable since operators want robots to
implement more missions with limited resources. Therefore, if a robot does not
choose an efficient strategy for completing a task, we can deduce that either the robot is not
rational or that our current model of the robot's goals are incorrect~\cite{best2015bayesian, ahmad2016bayesian, hwang2008intent,
yepes2007new, Fisac2018Probabilistically, fridovich2020confidence,
bajcsy2019scalable}.

\begin{figure}[t]
\begin{center}
  \begin{tabular}{cc}
\begin{tikzpicture}
\node (n0) at (0,0)
    {\includegraphics[width=.47\textwidth]{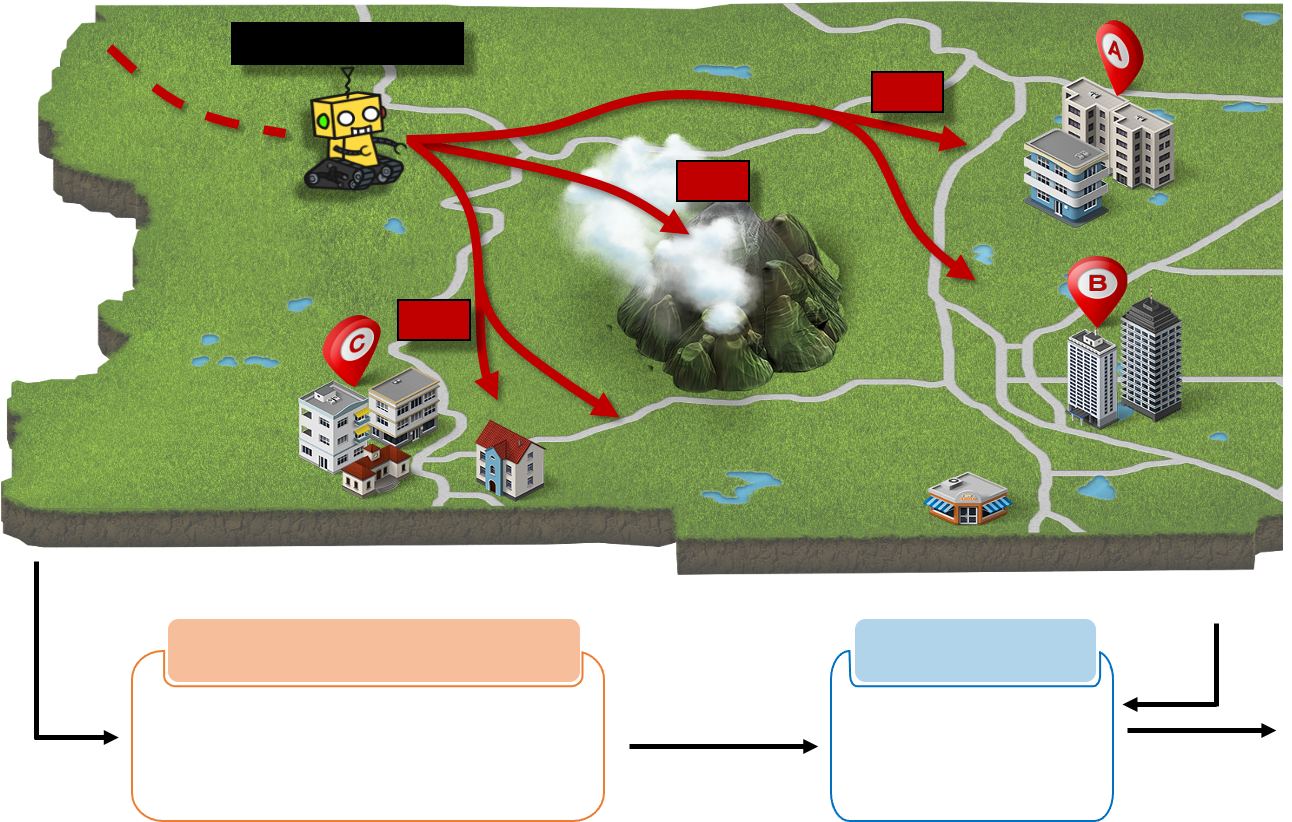}};
\node (n1) at (-1.75,-1.57) [font=\fontsize{7.2}{0}\selectfont] {\textbf{Hypothesis Generation}};
\node (n2) at (-1.8,-1.97) [font=\fontsize{6.5}{0}\selectfont] {H1: Avoid mountain, go to A};
\node (n3) at (-2.58,-2.2) [font=\fontsize{6.5}{0}\selectfont] {H2: Go to B};
\node (n4) at (-2.58,-2.45) [font=\fontsize{6.5}{0}\selectfont] {H3: Go to C};
\node (n5) at (-3.45,-1.05) [font=\fontsize{6.5}{0}\selectfont] {Env. Info};
\node (n6) at (0.5,-1.75) [font=\fontsize{6.5}{0}\selectfont] {Hypotheses};
\node (n7) at (0.5,-1.99) [font=\fontsize{6.5}{0}\selectfont] {Set};
\node (n8) at (2.15,-1.55) [font=\fontsize{7.2}{0}\selectfont] {\textbf{Evaluation}};
\node (n9) at (2.15,-1.97) [font=\fontsize{6.5}{0}\selectfont] {H1: 45\%};
\node (n10) at (2.15,-2.21) [font=\fontsize{6.5}{0}\selectfont] {H2: 35\%};
\node (n11) at (2.15,-2.45) [font=\fontsize{6.5}{0}\selectfont] {H3: 20\%};
\node (n12) at (3.53,-1.2) [font=\fontsize{6.5}{0}\selectfont] {Observation};
\node (n13) at (3.59,-2.26) [font=\fontsize{6.5}{0}\selectfont] {Trajectory};
\node (n13) at (3.59,-2.49) [font=\fontsize{6.5}{0}\selectfont] {Prediction};

\node (n14) at (-1.91,2.39) [font=\fontsize{6.5}{0}\selectfont, color=white] {Mobile Robot};
\node (n15) at (1.7,2.07) [font=\fontsize{6.5}{0}\selectfont, color=white] {H1};
\node (n16) at (0.44,1.5) [font=\fontsize{6.5}{0}\selectfont, color=white] {H2};
\node (n17) at (-1.37,0.59) [font=\fontsize{6.5}{0}\selectfont, color=white] {H3};
\end{tikzpicture}
  \end{tabular}
\end{center}
\caption{The proposed Bayesian intent inference approach first generates a hypotheses set (H1-H3) and uses the robot's recent positions to update the probability of various hypothesized intents (H1-H3). Finally, we can predict likely future positions of the robot using the inferred distribution over possible intents.}\label{fig:intro-illustration}
\end{figure}
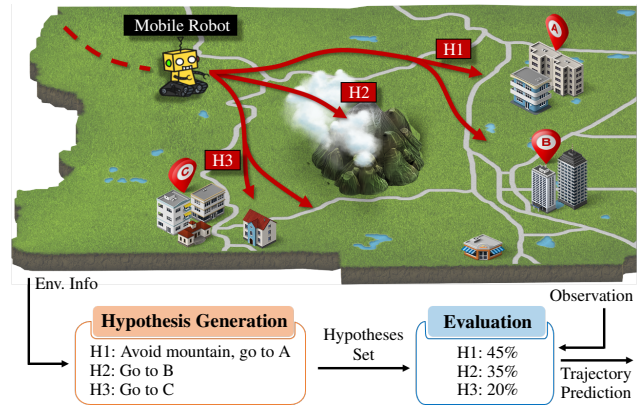

In this paper, we use the robot’s intent information for predictive runtime
monitoring. Our assumption is that, if a robot has a high-level mission, we are
able to (1) not only infer the mission by observing its behavior (2) but also
use the information to predict future positions. Therefore, the ability to find
the intent is key in our work.

To that end, we use temporal logic formulas to represent the intent. Temporal
logics have been quite popular for specifying missions in a precise manner and
generating efficient plans for carrying them out~\cite{bhatia2010sampling,
bhatia2010motion, fainekos2009temporal, kress2009temporal, humphrey2014formal,
vasile2014reactive, ulusoy2013optimality}. Temporal logics have been
demonstrated as suitable reprentations of complex real-world missions such as
surveillance and package delivery~\cite{lj2019priority,humphrey2014formal,
ferri2020cooperative, choudhury2020efficient}. We identify a subset of temporal
logic formulas corresponding to the safety and guarantee formulas in the
Manna-Pnueli hierarchy of temporal logic
formulas~\cite{Manna+Pnueli/1989/Hierarchy}. Such formulas can be satisfied or
violated by a finite prefix of an infinite sequence of actions and thus quite
suitable as representations of ``near-term''/``immediate'' intents suitable for
finite time horizon predictions of the robot's position. The Bayesian intent
inference framework then generates a finite set of possible intents using given
patterns of temporal logic formulas and places a prior distribution on these
formulas to represent the probability that a given formula represents the
robot's intent. Next, we use a model of ``noisy rationality'' to provide a
probability that a robot takes a given action in the workspace given its true
intent. This model compares the cost of the action and the most efficient path
from the resulting state to the overall goal of the intent against other
possible actions. We use temporal logic planning techniques based on converting
formulas to automata and solving shortest path problems to compute these costs.

Temporal logic specification inference from observation data have been studied
widely in the recent past~\cite{shah2018bayesian, kim2019bayesian,
vazquez2020maximum, vazquez2018learning}. The main difference from our work is
that they assume the entire trajectory is available at once, whereas we use the
parts of the trajectory. Furthermore, our approach uses intents as a means to
perform predictions of future positions.  

We evaluate our framework on two datasets: a probabilistic roadmap simulation
dataset, wherein we use the popular PRM planning technique to generate motion
plans for some tasks while using our intent inference technique to predict the
intents and future positions without knowledge of the overall mission plan. A
second data set consists of trajectories of humans inside a room, called
T{\"H}OR~\cite{thorDataset2019}: here we are provided noisy position
measurements with unknown intents. Thus, both datasets include a moving agent
implementing various subtasks on the way to a goal, which is unknown to our
monitor. The results show that our method can predict future positions with high
accuracy, and all computations can be implemented in real-time.

The contributions of this paper are as follows:
\begin{compactenum}
\item We introduce a Bayesian intent inference framework leveraging an intent information of a robot. The framework computes the probability distribution of all possible intents written in LTL. 
\item Using the outputs of the framework, we can effectively carry out predictive monitoring that can be used in many robotic applications.
\item All computations can be implemented with sufficient efficiency to enable real-time monitoring.
\end{compactenum}

To the best of our knowledge,  this work is the first attempt to use a logic-based Bayesian intent inference for predictive monitoring.  

\section{Problem Formulation}\label{sec:problem_formulation}
Central to our framework is a ``map'' of the robot's workspace that is
discretized into finitely many cells. Each cell is labeled with an atomic
proposition that characterizes the attributes of the cell. We use the
mathematical model of a \emph{weighted finite transition system} to capture the
map (or the workspace) of the robot.

\begin{definition}[Weighted Finite Transition System] 
A weighted finite transition system $\mathcal{T}$ is a tuple $(C, R, \Pi,
L, \omega)$ wherein $C$ is a finite set of cells,
$R \subseteq C \times C$ is the transition relation that
represents all allowable moves from one cell to the next by the robot, $\Pi$ is a set of boolean atomic
propositions, $L : C \to 2^{\Pi}$ is a labeling function that associates each cell $c \in C$
with a set of atomic propositions $L(c)$, and $\omega : R \to
\mathbb{R}_{\geq 0}$ maps each edge in $R$ to a
non-negative weight.
\end{definition}

Therefore, the position of a robot at time $t$ can be defined as a cell
$\mathbf{x}_t \in C$. Atomic propositions  label attributes/features
such as \emph{airport}, \emph{fire}, \emph{mountain}, and so on (see Fig.~\ref{fig:intro-illustration}).
%Furthermore, we may have a unique atomic proposition $\pi_i$ that provides a name for the
%corresponding cell $c_i$. 
A \emph{path} in $\scr{T}$ is an infinite sequence of cells $p=c_0 c_1
c_2 \cdots$ such that $c_i \in C$ and $(c_i, c_{i+1}) \in R$ for each
$i \in \mathbb{N}$.

\paragraph{Linear Temporal Logic}
In this paper, we assume that a robot has a high-level mission to implement
before going to a goal location. For example, ``$\mathbf{H_1}$:\ Visit $\pi_1, \pi_2$, and
$\pi_3$ in some order'', or ``$\mathbf{H_2}$:\ Visit $\pi_3$ while avoiding $\pi_5$''. To formally
express such requirements, we use linear temporal logic (LTL) whose grammar is
defined as follows:
$$ \varphi ::= \true\, |\, \false\, |\, \pi \in \Pi\, |\, \neg \varphi\,
|\, \varphi \wedge \varphi\, |\, \Circle \varphi \,
|\, \varphi\, \mathcal{U}\, \phi \,.$$ In addition, two temporal
operators, \emph{eventually} ($\lozenge \varphi:\ \true\,\mathcal{U}\,\varphi$)
and \emph{globally} ($\square \varphi:\ \neg\lozenge\neg\varphi$) can be
derived. The formula $\square \varphi$ is satisfied if $\varphi$ holds for all
time and $\lozenge \varphi$ is satisfied if eventually at some point in time
$\varphi$ is satisfied. We refer the reader to standard texts for a detailed
description of temporal logic and its
applications~\cite{Manna+Pnueli/92/Temporal,Baier+Katoen/2008/Principles}. Using
LTL, we can express the mission
$\mathbf{H_1}:\ \lozenge \pi_1 \wedge \lozenge \pi_2 \wedge \lozenge \pi_3$ and
$\mathbf{H_2}:\ \lozenge \pi_3 \wedge \square \neg \pi_5$. Using LTL is
beneficial because it is capable of describing complex missions clearly although
some fundamental properties like \emph{safety} ($\square \neg \varphi$)
and \emph{reachability} ($\lozenge \varphi$) are mostly used for robot missions
in many scenarios, and because it enables us to use \emph{temporal logic motion
planning}~\cite{bhatia2010sampling, bhatia2010motion, fainekos2009temporal,
kress2009temporal, humphrey2014formal, vasile2014reactive,
ulusoy2013optimality}.

\paragraph{Assumptions}
We assume full knowledge of the transition system $\mathcal{T}$ is available at
any time. Also, if the map is updated in the case of dynamic scenarios, the new
information is assumed to be available immediately. On the other hand, the
robot's mission is assumed to be unknown but expressible as a temporal logic
formula involving atomic propositions in the map.

In this paper, we investigate two problems --- intent inference and predictive monitoring. Fig.~\ref{fig:framework}
shows how these problems relate to each other in our proposed framework.

  \noindent \textbf{Intent Inference:} Given a transition system $\mathcal{T}$
  and the recent history of robot cells at time $t$,
  $\mathbf{x}_{t},\mathbf{x}_{t-1}, \cdots, \mathbf{x}_{t-h}$, we wish to infer
  a distribution of likely intents$\{ (\varphi_1, p_1), \ldots, (\varphi_n,
  p_n) \}$, wherein $\varphi_i$ is a temporal logic formula involving atomic
  propositions $\Pi$, and $p_i \geq 0$ is its associated probability with
  $\sum_{i=1}^n p_i = 1$.

  \noindent \textbf{Predictive Monitoring:} Given a distribution over intents,
  we wish to compute a distribution of future positions $\mathbf{x}_{t+k}$ at
  time $t+k$. At time $t+1$, our approach receives new robot position
  $\vx_{t+1}$, requiring updates to the intents, and the predicted future cell.
  This update needs to be computed in time that is much smaller
  than the overall sampling time.

\section{Bayesian Intent Inference}\label{sec:framework}
We first introduce our Bayesian approach to solve the intent inference problem.
The idea of our approach is to generate possible intents as our
\emph{hypotheses} and evaluate their probabilities using Bayesian inference (see
Fig.~\ref{fig:framework}).

\begin{figure}[t]
\begin{center}
\begin{tikzpicture}[font=\small]

% \draw (-4.2,-3) -- (-4.2,3) -- (4.2, 3) -- (4.2,-3) -- cycle;

\draw (-4.2,-3) -- (-4.2,2) -- (-1, 2) -- (-1,-3) -- cycle;
\node (n1) at (-2.6, 2.3) {\textbf{Hypothesis Generation}};

\draw [fill=green!20!white, draw=none] (-3.75,0.5) -- (-3.75,1.5) -- (-1.25, 1.5) -- (-1.25,0.5) -- cycle;
\node (n2) at (-2.5, 1.2) {Transition System};
\node (n3) at (-2.5, 0.8) {$\mathcal{T}$};
\node (n4) at (-2.9, 1.7) {Environment};
\draw (-3.75,1) node[inner sep=0pt, minimum size=0pt](pt_t) {};
\draw (-4,1) node[inner sep=0pt, minimum size=0pt](pt_tl) {};

\draw [fill=green!20!white, draw=none] (-3.75,-1.1) -- (-3.75,-0.1) -- (-1.25, -0.1) -- (-1.25,-1.1) -- cycle;
\node (n5) at (-2.5, -0.4) {B{\"u}chi Automata};
\node (n6) at (-2.5, -0.8) {$\mathcal{A}_{0 \cdots i}$};
\node (n7) at (-3.05, 0.1) {LTL Specs};
\draw (-2.5,-1.1) node[inner sep=0pt, minimum size=0pt](pt_b) {};

\draw [fill=yellow!50!white, draw=none] (-3.75,-2.7) -- (-3.75,-1.7) -- (-1.25, -1.7) -- (-1.25,-2.7) -- cycle;
\node (n8) at (-2.5, -2) {Product Automata};
\node (n9) at (-2.5, -2.4) {$ \mathcal{T} \otimes \mathcal{A}_{0 \cdots i} = \mathcal{P}_{0 \cdots i}$};
\draw (-2.5,-1.7) node[inner sep=0pt, minimum size=0pt](pt_p) {};
\draw (-3.75,-2.2) node[inner sep=0pt, minimum size=0pt](pt_pl) {};
\draw (-4,-2.2) node[inner sep=0pt, minimum size=0pt](pt_pl2) {};
\draw (-1.25,-2.2) node[inner sep=0pt, minimum size=0pt](pt_pr) {};
\draw (-0.75,-2.2) node[inner sep=0pt, minimum size=0pt](pt_pr2) {};
\draw (1.2,-2.3) -- (1.2,2) -- (4.2, 2) -- (4.2,-2.3) -- cycle;
\node (n10) at (2.65, 2.3) {\textbf{Hypothesis Evaluation}};

\draw [fill=blue!20!white, draw=none] (1.6,0.75) -- (1.6,1.55) -- (3.8, 1.55) -- (3.8,0.75) -- cycle;
\node (n10) at (2.7, 1.15) {Observation $\mathbf{x}_t$};
\draw (2.7,0.75) node[inner sep=0pt, minimum size=0pt](pt_ob) {};

\draw [fill=orange!20!white, draw=none] (1.45,0.3) -- (1.45,-0.7) -- (3.95, -0.7) -- (3.95,0.3) -- cycle;
\node (n10) at (2.7, 0) {Bayesian Inference};
\node (n10) at (2.7, -0.4) {$\mathbb{P}(\varphi_{0 \cdots i}|\mathbf{x}_t)$};
\draw (1.45,0) node[inner sep=0pt, minimum size=0pt](pt_ba) {};
\draw (2.7,0.3) node[inner sep=0pt, minimum size=0pt](pt_ba_up) {};
\draw (3.4,-0.7) node[inner sep=0pt, minimum size=0pt](pt_ba_br) {};
\draw (1.8,-0.7) node[inner sep=0pt, minimum size=0pt](pt_ba_bl) {};

\node (n10) at (3.4, -1.5) {Posterior};
\draw (3.4,-1.3) node[inner sep=0pt, minimum size=0pt](pt_post) {};
\draw (2.7,-1.5) node[inner sep=0pt, minimum size=0pt](pt_post_l) {};
\draw (3.4,-1.7) node[inner sep=0pt, minimum size=0pt](pt_post_down) {};
\node (n10) at (1.8, -1.3) {Post.};
\node (n10) at (1.8, -1.7) {Update};
\draw (1.8,-1.1) node[inner sep=0pt, minimum size=0pt](pt_po) {};
\draw (2.4,-1.5) node[inner sep=0pt, minimum size=0pt](pt_po_r) {};
\node (n10) at (3, -2.8) {Predict positions};
\draw (3.4,-2.5) node[inner sep=0pt, minimum size=0pt](pt_pre) {};
\draw [fill=blue!20!white, dashed] (-0.5,-2) -- (-0.5,1) -- (0.7, 1) -- (0.7,-2) -- cycle;
\node (n13) at (0.1, 1.4) {\textbf{Hypotheses}};
%\node (n14) at (0.1, 1.3) {\textbf{Set}};
\node (n14) at (0.1, 0.5) {$\varphi_0:\mathcal{P}_0$};
\node (n14) at (0.1, 0) {$\varphi_1:\mathcal{P}_1$};
\node (n14) at (0.1, -0.5) {$\vdots$};
\node (n14) at (0.1, -1.3) {$\varphi_i:\mathcal{P}_i$};
\draw (-0.5,0) node[inner sep=0pt, minimum size=0pt](pt_set) {};
\draw (-0.75,0) node[inner sep=0pt, minimum size=0pt](pt_set2) {};
\draw (0.7,0) node[inner sep=0pt, minimum size=0pt](pt_set_r) {};

\path[->] (pt_b) edge (pt_p)
(pt_t) edge[-] (pt_tl)
(pt_tl) edge[-] (pt_pl2)
(pt_pl2) edge (pt_pl)
(pt_set2) edge (pt_set)
(pt_pr) edge[-] (pt_pr2)
(pt_set2) edge[-] (pt_pr2)
(pt_set_r) edge (pt_ba)
(pt_ob) edge (pt_ba_up)
(pt_ba_br) edge (pt_post)
(pt_po) edge (pt_ba_bl)
(pt_po_r) edge[-] (pt_post_l)
(pt_post_down) edge (pt_pre);
\end{tikzpicture}
\end{center}
\caption{Diagram of the Bayesian intent inference framework}\label{fig:framework}
\end{figure}
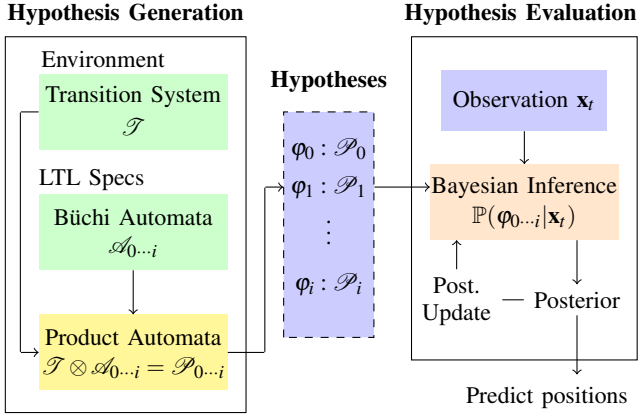

\subsection{Hypothesis Generation}\label{sec:hypo_gen}

Hypothesis generation is achieved using temporal logic \emph{specification
  patterns} that have been explored in previous works (Cf.
~\cite{humphrey2014formal,fainekos2009temporal}). Such patterns specify temporal
logic formulae with ``holes'' that can be filled in with atomic propositions.
Each such pattern defines a set of formulas obtained by substituting all
possible atomic propositions of interest for each hole. To avoid potentially
vacuous or inconsistent intents, we may further require that the same atomic
proposition not be used in two distinct holes for a given template.

\begin{example}\label{Ex:intent-patterns}
We list some commonly encountered patterns of interest below. We substitute an
atomic proposition in the place of a hole denoted by ``$\hole$'', ensuring that
the same proposition does not appear in more than one hole.

\begin{itemize}
\item \emph{Avoid Region}: $\square \neg \hole$
\item \emph{Cover Region}: $\lozenge \hole$
\item \emph{First and Then Second Region}: $\lozenge \left(\hole \wedge \lozenge \left(\hole  \right)\right)$
\item \emph{Reach While Avoid}: $\lozenge \hole \wedge \square \neg \hole$
\end{itemize}
\end{example}

As a result, each pattern can be expanded out into a set of LTL formulae that represent possible intents
of the agent.

\subsection{Temporal Logic and B\"uchi Automata}

We recall the standard connection between temporal logics and automata on
infinite strings, specifically B\"uchi automata~\cite{Wolper/2002/Constructing,Thomas/1990/Automata}.
Let $\varphi$ be a temporal logic formula over atomic propositions in $\Pi$.
Recall such a formula can be encoded as a nondeterministic B\"uchi automaton.

\begin{definition}[B\"uchi Automaton]
  A B\"uchi automaton $\mathcal{A}$ is a tuple $(Q, \Pi, E, q_0, F)$ wherein $Q$
  is a finite set of states; $\Pi$ is a finite set of atomic propositions; $E
  \subseteq Q \times \Pi \times Q$ is a set of transitions, wherein each
  transition $(q_i, \pi, q_j)$ indicates the transition from state $q_i$ to
  $q_j$ upon observing atomic proposition $\pi$; $q_0$ is an initial state and
  $F$ is the set of accepting state.
\end{definition}

Given an infinite sequence of atomic propositions $\pi_0, \pi_1, \pi_2, \ldots$,
a run of the automaton is an infinite sequence of states $q_0, q_1, q_2,
\ldots$, such that $q_0$ is the initial state and $(q_i, \pi_i, q_{i+1}) \in E$
for all $i \geq 0$. Finally, a run is accepting iff it visits an accepting state
$q \in F$ infinitely often. It is well-known that every LTL formula can be
translated into a B\"uchi
automaton~\cite{Manna+Pnueli/92/Temporal,Baier+Katoen/2008/Principles}. The
problem of constructing a B{\"u}chi automaton from a LTL specification has been
widely studied~\cite{gastin2001fast} with numerous tools such as
SPOT~\cite{duret.16.atva2}.

\paragraph{Safety/Guarantee Formulas and Automata}

In this paper, we focus on a very specific class of safety and guarantee
formulas, originally introduced by Manna \& Pnueli as part of a larger
classification of all LTL formulas~\cite{Manna+Pnueli/1989/Hierarchy}. Briefly,
safety formulas can be written using the $\square$ operator with negations
appearing only in front of atomic propositions, whereas guarantee formulas are
written using the $\lozenge$ operator with negations appearing only in front of
atomic propositions.

\begin{example}
  Going back to the Example~\ref{Ex:intent-patterns}, we note that the ``avoid
  regions'' pattern is a safety formula, whereas the ``cover regions'' and
  ``temporal sequencing'' patterns are guarantee formulas. Note that the
  coverage with the safety pattern is the conjunction of a guarantee sub-formula
  (involving $\lozenge$) and a safety sub-formula (involving $\square$).
\end{example}

\noindent\textbf{Assumption:} We will assume that any hypothesis being
considered can be written as
\begin{equation}\label{eq:formula-pattern}
  \left(\bigwedge_{i=1}^{M}\ \square \neg \pi_{s,i}\right)\  \land \ \left( \bigwedge_{j=1}^{N}\ \lozenge \pi_{g,j}\right) \,,
\end{equation}
wherein $A: \{ \pi_{s,1}, \ldots, \pi_{s,M}\}$ is disjoint from $B: \{
\pi_{g,1}, \ldots, \pi_{g,N} \}$, and $N > 1$ (i.e, $B \not= \emptyset$). Such a
formula represents the intent that the robot seeks to reach all regions labeled
by atomic propositions in the set $B$, in some order, while avoiding all regions
in A. More generally, however, our framework can accommodate the conjunction of
safety formulas and guarantee formulas. 

However, since our framework is \emph{probabilistic} it associates 
a measure of belief/probability with each hypothesis. Also, since our framework is \emph{dynamic}, these
probabilities  change over time. Thus, it is possible for our framework to implicitly infer
a more complex high  level objective that is not expressible in our restricted fragment of LTL. 
We will explore this aspect of our work further in the future.

We now consider a special type of B\"uchi automaton that we will call a
\emph{safety-guarantee} automaton.

\begin{definition}[Safety-Guarantee Automaton]
  A B\"uchi automaton is said to be a safety-guarantee automaton if the set of
  states $Q$ is partitioned into three mutually disjoint parts: $Q:\ Q_t \uplus
  F \uplus \{r \}$ wherein (a) the initial state $q_0 \in Q_t \cup F$, (b) $Q_t$ is a set 
  of ``transient'' states such that no state in $Q_t$ is accepting; (c) $F$
  is the set of accepting states, and (d) $r$ is a special \emph{reject state}.
  Furthermore, the outgoing edges from each state in $F$ either take us to a
  state in $F$ or to the reject state $r$. Finally, all outgoing edges from $r$
  are self-loops back to $r$. Fig.~\ref{fig:safety-guarantee-aut} illustrates
   safety-guarantee automata.
\end{definition}

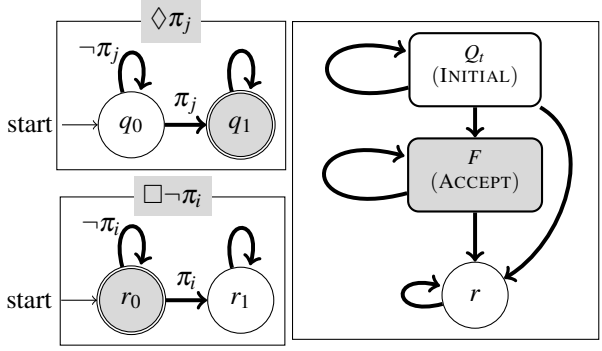
\begin{figure}[t]
  \begin{center}
    \begin{tikzpicture}
      \begin{scope}
        \matrix[row sep=40pt, column sep=15pt]{
          \node[state,initial](n0){$q_0$}; & \node[state,accepting,fill=gray!30](n1){$q_1$};\\
          \node[state,initial,accepting,fill=gray!30](m0){$r_0$}; & \node[state](m1){$r_1$};\\
        };
        \path[->, line width=1.5pt] (n0) edge [loop above] node[left](n2){$\neg \pi_j$}(n0)
        (n0) edge node[above]{$\pi_j$} (n1)
        (n1) edge[ loop above] (n1)
        (m0) edge [loop above] node[left](m2){$\neg \pi_i$}(m0)
        (m0) edge node[above]{$\pi_i$} (m1)
        (m1) edge[ loop above] (m1);
      \end{scope}
    \begin{scope}[xshift=4.5cm,yshift=1.9cm]
      \node[rectangle, draw=black, thick, rounded corners](pp0) at (0,0) {\footnotesize $\begin{array}{c} Q_t \\ (\textsc{Initial})\\[3pt] \end{array}$};
      \node[rectangle, draw=black, thick, rounded corners,fill=gray!30](pp1) at (0,-1.4) {\footnotesize $\begin{array}{c}  F \\ (\textsc{Accept})\\[3pt] \end{array}$};

      \node[state](pp2) at (0,-3) {$r$};

      \path[->, line width=1.5pt](pp0) edge (pp1)
      (pp1) edge (pp2)
      (pp1) edge[loop left] node[left](pp3){} (pp1)
      (pp0) edge[loop left] node[left](pp4){} (pp0)
      (pp0) edge[out=330, in=30] node[right](pp5){} (pp2)
      (pp2) edge[loop left] (pp2);

    \end{scope}
    \begin{scope}[on background layer]
      \node[fit=(n0)(n1)(n2),inner sep=4pt, draw=black, rectangle](nn0){};
      \node[rectangle,fill=gray!30] at (nn0.north) {$\lozenge \pi_j$};
      \node[fit=(m0)(m1)(m2),inner sep=4pt, draw=black, rectangle](nn1){};
      \node[rectangle,fill=gray!30] at (nn1.north) {$\square \neg \pi_i$};
      \node[fit=(pp0)(pp1)(pp2)(pp3)(pp5), inner sep=4pt, rectangle,draw=black](nn2){};
    \end{scope}
    \end{tikzpicture}
  \end{center}

  \caption{\textbf{(left)} B\"uchi automata for $\lozenge \pi_j$ and $\square \neg \pi_i$; and \textbf{(right)} Overall structure
    of a safety-guarantee automaton.}\label{fig:safety-guarantee-aut}

\end{figure}

\begin{lemma}
  A formula that satisfies the pattern in Eq.~\eqref{eq:formula-pattern} is
  represented by a safety-guarantee B\"uchi automaton.
\end{lemma}
\begin{proof}(Sketch)
  Note that such a formula is made up of a conjunction of $\square (\neg \pi_j)$
  and $\lozenge \pi_i$ subformulas whose automata are shown in
  Fig.~\ref{fig:safety-guarantee-aut} (left). The overall conjunction is
  represented by the product of these automata, wherein a product state is
  accepting iff each of the individual component states are accepting. The rest
  of the proof is completed by identifying the states in each partition to
  establish the overall safety-guarantee structure of the automaton.
\end{proof}

\paragraph{Significance of Safety-Guarantee Structure} We will briefly explain
why the overall structure of the automaton is important in our framework. Note
that temporal logic formulas are quite powerful in expressing a variety of
patterns that may include formulas such as $\square \lozenge \pi$ which states
that a cell satisfying the atomic proposition $\pi$ must be reached infinitely
often, or $\lozenge \square \pi$ which states that the robot will eventually
enter a region where $\pi$ holds and stay in that region forever. Natually, it
is \emph{impossible} us to infer that such an intent holds or otherwise by
observing any finite sequence of cells, no matter how long such a sequence may
be. For instance, a robot intending to visit a region infinitely often may take
a long time before its first visit to such a region since there are infinitely
many steps ahead in the future. In this regard, the safety-guarantee structure
allows the robot to signal its likely intent using a finite sequence: a robot
intending to satisfy an intent can signal this in finitely many steps by
reaching an accepting state in $F$. Likewise, a violation can also be seen in
finitely many steps by reaching the reject state $r$. 

\paragraph{Product Automaton}
We define the Cartesian product  between a
weighted transition system $\scr{T}$ defining the workspace  and a
B\"uchi automaton $\scr{A}$.

\begin{definition}[Product Transition System]
  The product automaton $ \mathcal{T} \otimes \mathcal{A}$ is defined as the
  tuple: $(S, \delta, \hat{F}, \hat{\omega})$:
\begin{compactenum}
  \item $S :\ C \times Q$ is the
  Cartesian product of the set of cells in $\scr{T}$ and states in $\scr{A}$;
  \item $\delta \subseteq S \times S $ is a transition relation s.t. $((c_i, q_i),
  (c_j, q_j)) \in \delta$ iff  $(c_i, c_j)\in R$ and $(q_i, \pi_k, q_j) \in E$
    for some $\pi_k \in L(c_i)$;
   \item $\hat{F}:\ C \times F$ is the set of accepting
     states, and
   \item $\hat{\omega}((c_i, q_i), (c_j, q_j))$ is a weight function that
  is set to be equal to $\omega(c_i, c_j)$ if $((c_i, q_i), (c_j, q_j)) \in
  \delta$
  \end{compactenum}
\end{definition}

The product automaton models all the ``joint'' moves that can be made by a copy
of the automaton $\scr{A}$ in conjunction with a transition system $\scr{T}$,
wherein the atomic propositions labeling each cell in $\scr{T}$ governs the
possible enabled edges in the automaton $\scr{A}$.

\subsection{Cost of Formula Satisfaction}

Let $\scr{T}$ be a weighted transition system describing the workspace of the
robot and $\psi$ be a formula that follows the pattern in
Eq.~\eqref{eq:formula-pattern}, and described by a safety-guarantee automaton
$\scr{A}_{\psi}$. For a given state $\vx_t$ of $\scr{T}$, we define the cost of
satisfaction: $\cost(\vx_t, \varphi)$ as the shortest path cost for a path in
the transition system $\scr{T}$ whose atomic propositions satisfy the formula
$\varphi$. Formally, we define (and compute) $\cost(\vx_t, \varphi)$ using the
following steps:
\begin{enumerate}
\item Compute the product automaton $\scr{T} \otimes \scr{A}_{\psi}$.
\item Compute the shortest path cost from the product state $(\vx_t, q_0)$ to
  the set of accepting states $\hat{F}$ in the product automaton, wherein the
  cost of a path is given by the some of edge weights along the path.
\end{enumerate}

Note that the shortest cost path from a single product automaton state to a set
of accepting states is defined as the minimum among all possible shortest path
from the source to each element of the set. Since all edge weights are positive,
we can calculate the cost from each cell $\vx_t \in C$ to the set of accepting
states in time using Dijkstra's algorithm (single destination shortest path). To
handle a set of possible destination, we simply add a designated new destination
node and connect all accepting states to it using a 0 cost edge. This
calculation runs in time $O( (|\delta| + |S|) \log(|S|))$ wherein $|S| = |C|
\times |Q|$ is the number of states in the product automaton and $|\delta| = |R|
\times |E|$ denotes the number of edges.

\subsection{Bayesian Inference of Intent}

Let $\scr{H}:\{\varphi_1, \ldots, \varphi_n\}$ be
the set of hypothesized intents of the robot whose current cell is denoted by
$\vx_t \in C$. We will assume a prior probability distribution $\pi$ over $\scr{H}$
wherein $\pi(\varphi_j)$ denotes the prior probability over hypothesis formula
$\varphi_j$. Our initial prior starts out by assigning each hypothesis a uniform
probability. The posterior
from step $t-1$ forms the prior for step $t$ with some modifications.

At each step, we obtain an updated robot position $\vx_{t+1} \in C$ and use this fact to update the current distribution over $\scr{H}$. To do so, we require a model of robot decision making that determines the conditional probability $\Pp(\vx_{t+1}\ |\ \vx_t, \varphi)$:  the probability  given the intent $\varphi$ and 
current cell $\vx_t$, the robot moves to cell $\vx_{t+1}$. We will  make an assumption of \emph{Boltzmann noisy
  rationality}~\cite{ziebart2008maximum}.
\begin{figure*}[t]
\begin{tikzpicture}
\node (n0) at (0,0) {\includegraphics[width=0.98\textwidth]{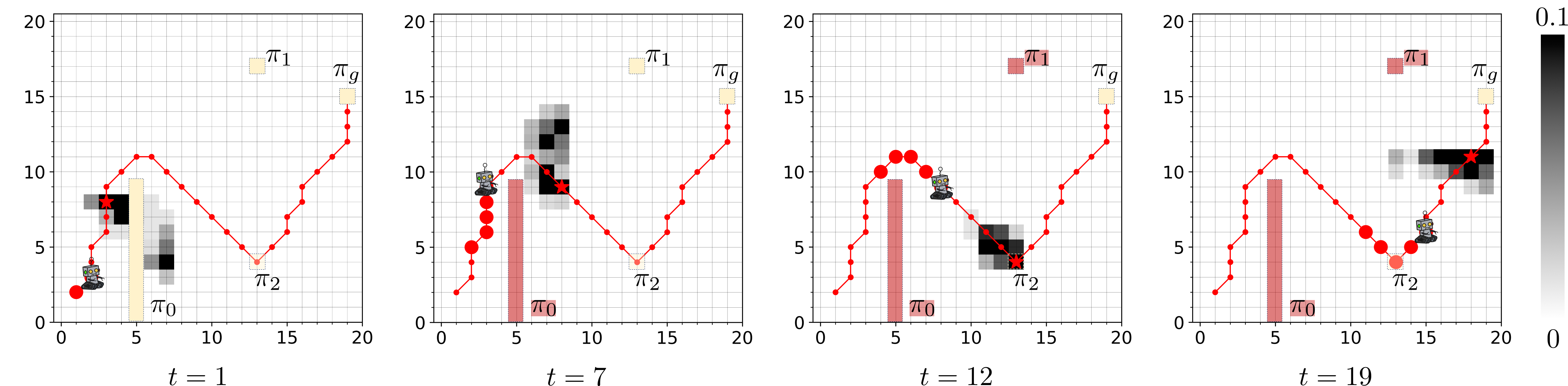}};
\end{tikzpicture}
\caption{Example scenario of predictive monitoring for the robot that has an underlying 
intent $\square \neg \pi_0 \wedge \lozenge \pi_2 \wedge \lozenge \pi_g$. 
The past five states (red circles) are used for Bayesian intent inference and our monitor
computes a distribution of future states (gray). The future trajectory is shown using small dots
and the ground truth of the future state is represented by red stars. As the robot navigates on the map, 
the monitor found $\pi_0$ and $pi_1$ are not a part of goals.}\label{fig:example}
\end{figure*}

\paragraph{Boltzmann Noisy Rationality Model}
Let $\nx(\vx_t)$ denote all the neighboring cells to $\vx_t$. We assume that
for each cell $c \in \nx(\vx_t)$ the probability of moving to $c$ is proportional an exponential of the
sum of the cost of moving from $\vx_t$ to $c$ and the cost of achieving the goal from $c$.
\[ \Pp(c | \vx_t, \varphi_j) \ \propto\ \exp\left( -\beta ( \omega(\vx_t, c) + \cost(c, \varphi_j)) \right) \,, \]
wherein $\beta$ is a chosen positive number that represents the rationality. For $\beta =0$, the robot's
choice is just a uniform choice between all available moves regardless of the intent. However as
$\beta \rightarrow \infty$, the agent simply chooses the optimal edge along the shortest cost path. 
With the appropriate normalizing constant, we note that

\begin{equation}\label{eq:boltzmann}
  \Pp(\vx_{t+1} | \vx_t, \varphi_j) = \frac{\exp\left( -\beta\ ( \omega(\vx_t, \vx_{t+1}) + \cost(\vx_{t+1}, \varphi_j) ) \right)}{ \sum_{c \in \nx(\vx_t)} \exp\left( -\beta\  ( \omega(\vx_t, c) + \cost(c, \varphi_j) ) \right) } \,.
  \end{equation}

Using Bayes rule, we can now compute the (unnormalized) posterior likelihood as follows:
\begin{equation}\label{eq:posterior-likelihood}
  \Pp(\varphi_j\ |\vx_t, \vx_{t+1})\ \propto\ \Pp(\vx_{t+1}\ |\ \vx_t, \varphi_j)\ \times\ \pi(\varphi_j) \,.
\end{equation}
The posterior probability is calculated by normalizing this over all
hypothesized intents $\varphi \in \scr{H}$.
We will recursively update the prior at each step to yield the posterior at the next step. However, it is
often useful to capture a change in the intent at each step by means of an ``$\epsilon$-transition'':
\begin{equation}\label{eq:epsilon-transition}
  \pi_{t+1} (\varphi_j) = ( 1- \epsilon) \Pp(\varphi_j\ |\ \vx_t, \vx_{t+1}) + \epsilon \frac{1}{|\scr{H}|} \,.
\end{equation}
The so-called $\epsilon$ transition simply weights down the posterior by a
factor $1 - \epsilon$ and adds a uniform probability distribution with a
constant weight $\epsilon$.  This method allows us to quickly capture the new intent when an agent changes its intent during the operation.  In our experiments, we fix $\epsilon = 0.3$.

\subsection{Posterior Predictive Distribution}

Given the current posterior $\Pp(\varphi_j | \vx_t, \vx_{t+1})$ computed using
Eq.~\eqref{eq:posterior-likelihood}, we wish to forecast the future position of
the robot. We model the movement of the robot as a stochastic process:
\begin{compactenum}
\item Let initial position be $\vx_{t+1}$ and the initial intent distribution be given by the distribution
  $\pi_{t+1}$ from Eq.~\eqref{eq:epsilon-transition}.
\item At  time $t = t +k$, update the current intent distribution:
  $\pi_{t+k+1} = (1- \epsilon) \pi_{t+k} + \epsilon\ \mathsf{Uniform}(\scr{H})$,
  wherein $\mathsf{Uniform}(\scr{H})$ represents a uniform distribution over the elements of the finite set $\scr{H}$.
\item Sample an intent $\varphi_j$ from $\pi_{t+k+1}$.
\item Sample $\vx_{t+k+1}$ from the distribution $\Pp(\vx_{t+k+1} | \vx_{t+k}, \varphi_j)$ according to Eq.~\eqref{eq:boltzmann}.
\end{compactenum}

Using the procedure above, we obtain samples of potential future trajectories of
the robot. We can use these trajectories to predict the possible future
positions at a future time $t+T$. Notice, however, that our model implicitly
assumes that the robot's intents change arbitrarily and are chosen afresh at
each step according to the posterior distribution.

\paragraph{Example Scenario}

Fig.~\ref{fig:example} shows an example scenario of a robot in a workspace with
four distinct regions labeled with atomic propositions $\pi_0, \pi_1, \pi_2$ and
$\pi_g$ as shown by the highlighted rectangles in the figure. The underlying
(ground truth) intent is to avoid the region $\pi_0$, visit regions $\pi_2$ and
$\pi_g$. The path taken by the robot is shown using the red circles whereas the
predicted future distribution is shown using various shades of gray, the darker
shade representing a higher probability. The ground truth future position is
shown using the red star.

At time $t=1$, having observed just two data points, the intent to avoid $\pi_0$
is guessed by our monitor.
However, at time $t=7$, the robot is seemingly unable to distinguish between the
competing goals of reaching/avoiding $\pi_1, \pi_g$ and $\pi_2$. However, at this
time, the monitor predicts a
right turn with a high probability though the robot's direction of travel would
indicate that it continues moving in a straight line in the positive $y$
direction.

At time $t=12$, we see that the robot's direction of travel makes the intent to
reach $\pi_2$ clear. Similarly, at time $t=19$, we see that the goal of reaching
$\pi_1$ or that of $\pi_g$ are considered likely with $\pi_g$ being seen as more
likely to be the robot's intended target.

Thus, we see how the robot's future positions are predicted accurately by our
monitor even though the set of possible intents are restricted to simple
safety-guarantee formulas.

%\section{Predictive Runtime Monitoring}\label{sec:predictive_rv}
%\input{predictive_rv}

\section{Experimental Results}\label{sec:experiments}
In this section, we evaluate the performance of our monitoring approach on two
datasets -- a synthetic data set generated by a Probabilistic Roadmap (PRM)  motion planning algorithm
that was used to plan paths satisfying randomly generated ground truth ``intents'',
and the TH{\"O}R human trajectory dataset~\cite{thorDataset2019}. Using these, we will answer the
following questions: (Q1) How does the prediction accuracy change as a prediction horizon increases?
and (Q2) How does the computation time
depend on the prediction time horizon?

Each evaluation is performed over a map with $K$ distinguished regions marked by atomic
propositions $\pi_1, \ldots, \pi_K$. The hypothesized intents consist of $2^K$ formulas, each of the
form $\bigwedge_{j \in A} \lozenge \pi_j\ \land\ \bigwedge_{i \in B} \square \neg \pi_i$ wherein
$A \cap B = \emptyset$ and $ A \cup B = \{ 1, \ldots, K\}$. %% Our implementation uses the fact
%% that the automata for these formulas have the same structure but different finite states
%% to save time computing the B\"uchi automata and the cartesian product between the map and the automaton.
Each evaluation consists of a path followed by the robot wherein the monitor
predicts the probability distribution of the positions $5, 10$ and $15$ steps
ahead. A prediction is deemed correct if the actual ground-truth
position is predicted by our monitor as having a probability $\geq 0.01$.

The trajectories are generated using two approaches: 

\noindent\textbf{PRM trajectories:} We use a  map with $N \times N$ grid cells and place $K$ random regions on the map.
Next, we mark a randomly chosen subset of these regions as obstacles and select
the remaining regions as targets. We use the PRM motion planner off the shelf to
generate a plan that may not necessarily be optimal, but is often close to being
optimal. Our experiments vary $N \in \{ 20, 50, 100 \}$ and $K \in \{3, 5\}$.

\noindent\textbf{TH\"OR Human Motion Dataset:}
This publicly available dataset includes multiple human trajectory data recorded
in a room of $8.4 \times 18.8$ meters~\cite{thorDataset2019}. The workspace
consists of five goal locations around the room and one obstacle in the middle.
Participants navigate between goals while avoiding the obstacle. To use this
dataset for our monitor, we converted the workspace to a $50 \times 50$ grid map,
and discretized human trajectory data. Among all trajectories, we selected $277$
segments at random for our evaluation.

\begin{figure}[t]
\begin{center}
  \begin{tabular}{cc}
\begin{tikzpicture}
\node (n0) at (0,0)
    {\includegraphics[width=.47\textwidth]{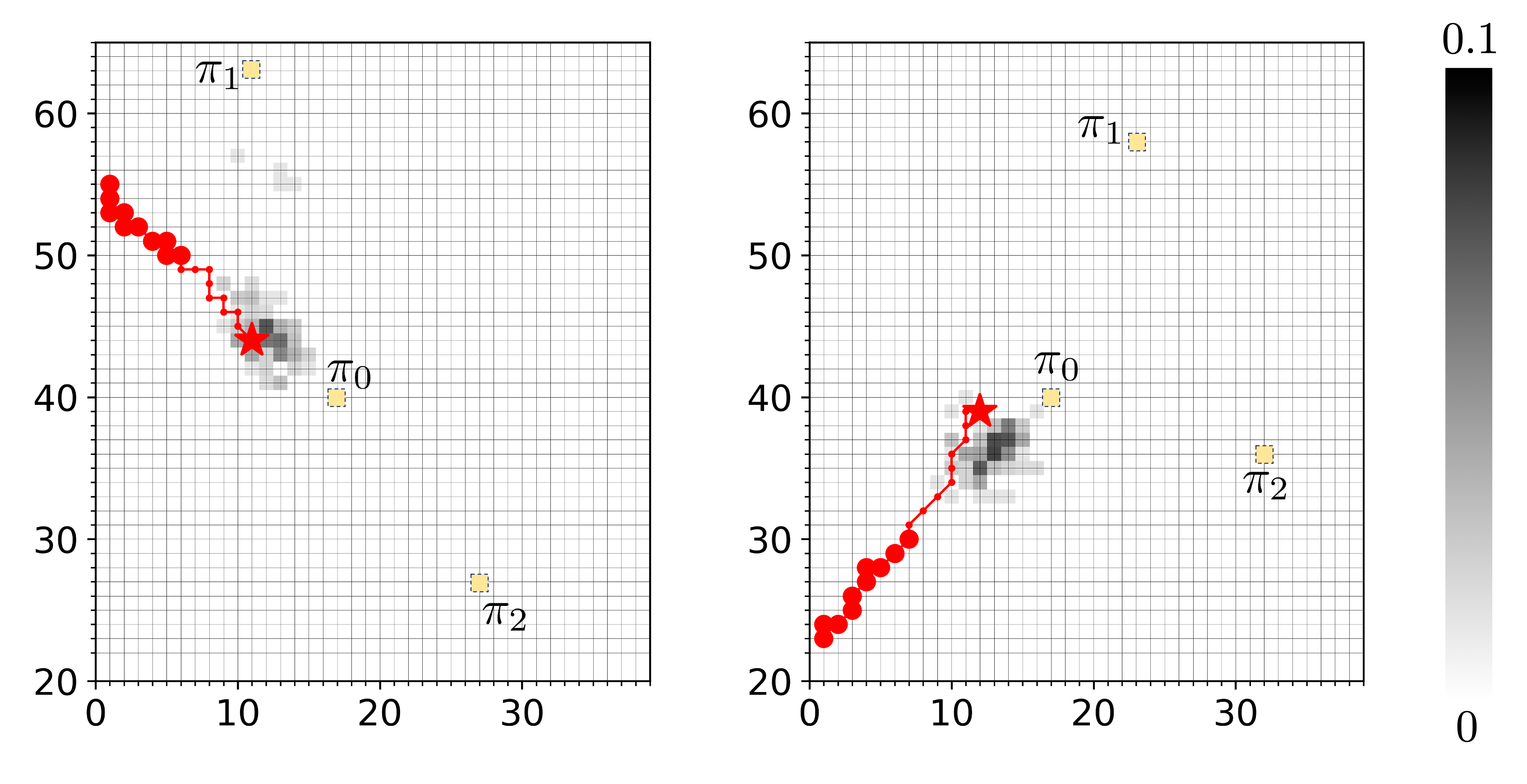}};
\end{tikzpicture}
  \end{tabular}
\end{center}
\caption{Example of the evaluation setup showing  goal regions labeled by atomic
propositions and an underlying intent that seeks some subset of the
regions/avoids the rest. The revealed trajectory is shown using red circles
while the future trajectory is shown using dotted lines. Predicted distribution
of future states is shown in various shades of gray with darker shades
representing higher probabilities.}\label{fig:prm}
\end{figure}

\begin{figure}[t]
\begin{center}
  \begin{tabular}{cc}
\begin{tikzpicture}
\node (n0) at (0,0)
    {\includegraphics[width=.44\textwidth]{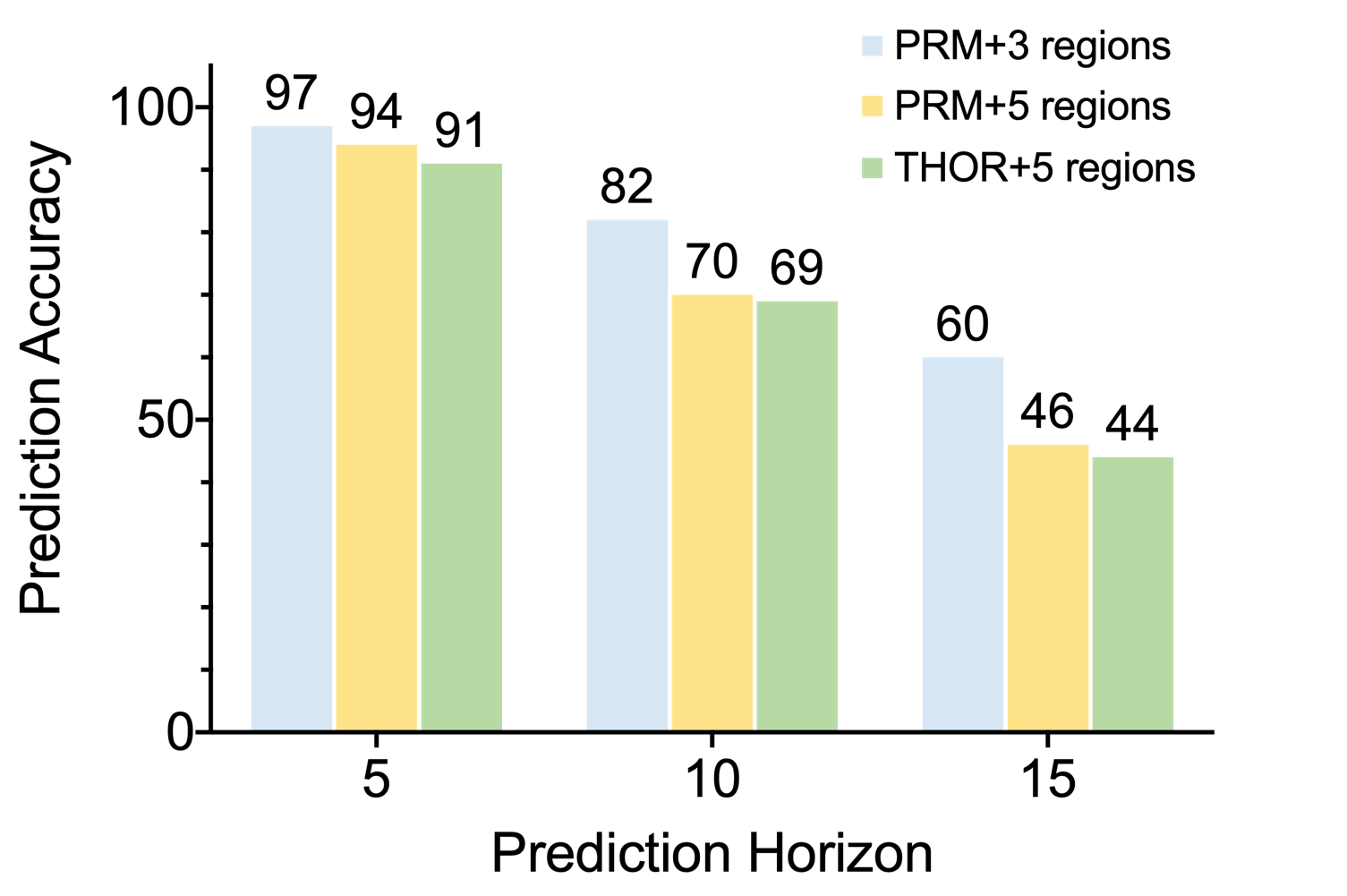}};
\end{tikzpicture}
  \end{tabular}
\end{center}
\caption{Prediction accuracy tested on various settings.}\label{fig:accuracy}
\end{figure}

\subsection{Evaluation of Accuracy and Computation Time}
Fig.~\ref{fig:accuracy} shows that the prediction accuracy for the various
datasets under varying prediction horizons. We note that the performance
degrades as the prediction horizon grows, as expected. Nevertheless, our
approach continues to provide useful information nearly $70\%$ of the time on
prediction horizons that are $10$ steps ahead. Furthermore, this accuracy can be
increased further by considering correlations of intents over time which is
currently not performed in our approach. We also note that the accuracy degrades
when more atomic propositions are considered and thus more hypothesized intents
are available. Finally, we notice that the accuracy for the real-life human
trajectory dataset is comparable with that of the synthesized PRM dataset. This
partly validates the rationality hypothesis that
underlies our work.

Next, we consider an evaluation of the computation time. The computational
complexity of our approach dependes on the size of product automata and the
number of hypotheses. Table~\ref{tab:computation-time} reports the computation
time for constructing product automata, checking $32$ intents, and Monte-Carlo
simulations for various map sizes. We note that
the computation times remain small even for a $100 \times 100$ grid and $32$ intents
each with $32$ B\"uchi automaton states. 

\begin{table}[t]
\caption {Computation time for our approach, as recorded 
on a MacBook Pro with 2.6 GHz Intel Core i7 and 16 GB RAM.
} \label{tab:computation-time}
\centering
\begin{tabular}{c|ccc}
\hline
Map size                                                                                                                       & 20$\times$20 & 50$\times$50 & 100$\times$100 \\ \hline
\begin{tabular}[c]{@{}c@{}}Product Automaton Construction\\ (32 states in a B{\"u}chi automaton)\end{tabular} & 0.07         & 0.23         & 0.75           \\ \hline
\begin{tabular}[c]{@{}c@{}}Bayesian Intent Inference\\ with 32 hypotheses\end{tabular}                                         & 0.16         & 0.56         & 2.13           \\ \hline
\begin{tabular}[c]{@{}c@{}}300 Monte-Carlo Simulations\\ (5 / 10 / 15 steps)\end{tabular}                                      & \multicolumn{3}{c}{0.28 / 0.55 / 0.81}      \\ \hline
\end{tabular}
\end{table}

\section{Conclusion}
Thus, we have demonstrated a framework for inferring intents and predicting likely future positions of robots. Our framework can be extended in many ways including richer set of
intents,  alternative assumptions on how intents change over time, incorporating richer agent dynamics, maps with time-varying regions of interest, alternatives to the noisy rationality model considered, and finally, intents governing multiple agents. We propose to study these problems using the rich framework of logic, automata and games combined with fundamental insights from Bayesian inference and machine learning.

\section*{ACKNOWLEDGMENTS}
 This work was funded in part by the US National Science
Foundation (NSF) under award numbers 1815983, 1836900 and the NSF/IUCRC Center for Unmanned Aerial Systems (C-UAS).  The authors thank Prof. Morteza Lahijanian for helpful discussions.

%%%%%%%%%%%%%%%%%%%%%%%%%%%%%%%%%%%%%%%%%%%%%%%%%%%%%%%%%%%%%%%%%%%%%%%%%%%%%%%%
% ---- Bibliography ----

\end{document}